\documentclass[lettersize,journal]{IEEEtran}
\usepackage{amsmath,amsfonts}
\usepackage{array}
\usepackage[caption=false,font=normalsize,labelfont=sf,textfont=sf]{subfig}
\usepackage{textcomp}
\usepackage{stfloats}
\usepackage{url}
\usepackage{verbatim}
\usepackage{graphicx}
\usepackage{cite}
\usepackage{microtype}
\usepackage{graphicx}
\usepackage{booktabs}
\usepackage{tabularx}
\usepackage[bookmarks=true,citecolor=blue]{hyperref} 
\hypersetup{colorlinks=true,linkcolor=blue}

\usepackage{caption} 
\usepackage{algorithm}
\usepackage[noend]{algorithmic}

\newlength\myindent
\newcommand\bindent{%
  \begingroup
  \setlength{\itemindent}{\myindent}
  \addtolength{\algorithmicindent}{\myindent}
}
\newcommand\eindent{\endgroup}

\usepackage{amsmath}
\usepackage{amssymb}
\usepackage{mathtools}
\usepackage{amsthm}

\usepackage[capitalize,noabbrev]{cleveref}

\theoremstyle{plain}
\newtheorem{theorem}{Theorem}[section]

\newtheorem{informalprop}[theorem]{Informal Proposition}

\theoremstyle{definition}
\newtheorem{definition}[theorem]{Definition}

\theoremstyle{remark}
\newtheorem{remark}[theorem]{Remark}
\DeclareMathOperator{\Tr}{Tr}

\usepackage{thmtools}
\usepackage{thm-restate}

\usepackage[textsize=tiny]{todonotes}

\usepackage{enumitem}
\usepackage{xspace}
\usepackage[numbers,square,compress]{natbib}

\newcommand{\defeq}{\vcentcolon=}

\usepackage{dsfont}
\usepackage{stmaryrd}
\usepackage{verbatim}

\renewcommand{\phi}{\varphi}

\newcommand{\half}{\tfrac{1}{2}}

\newcommand{\R}{\mathbb{R}}

\newcommand{\range}{\operatorname{range}}
\newcommand{\proj}{\operatorname{proj}}
\newcommand{\ortho}{\operatorname{orth}}

\newcommand{\rank}{\operatorname{rank}}

\DeclareMathOperator*{\argmin}{arg\,min}

\newcommand{\rev}[1]{#1}

\title{SketchOGD: Memory-Efficient Continual Learning}

\author{%
  Youngjae Min, Benjamin Wright, Jeremy Bernstein, \textnormal{and} Navid Azizan\\
  Massachusetts Institute of Technology \\  \texttt{\{yjm,bpwright,jbernstein,azizan\}@mit.edu} \\
}

\begin{document}

\maketitle

\begin{abstract}
When machine learning models are trained continually on a sequence of tasks, they are often liable to forget what they learned on previous tasks---a phenomenon known as \textit{catastrophic forgetting}. Proposed solutions to catastrophic forgetting tend to involve storing information about past tasks, meaning that memory usage is a chief consideration in determining their practicality. This paper develops a memory-efficient solution to catastrophic forgetting using the idea of \emph{matrix sketching}, in the context of a simple continual learning algorithm known as orthogonal gradient descent (OGD). OGD finds weight updates that aim to preserve performance on prior datapoints, using gradients of the model on those datapoints. However, since the memory cost of storing prior model gradients grows with the runtime of the algorithm, OGD is ill-suited to continual learning over long time horizons. To address this problem, we propose SketchOGD. SketchOGD employs an online sketching algorithm to compress model gradients as they are encountered into a matrix of a fixed, user-determined size. In contrast to existing memory-efficient variants of OGD, SketchOGD runs online without the need for advance knowledge of the total number of tasks, is simple to implement, and is more amenable to analysis. We provide theoretical guarantees on the approximation error of the relevant sketches under a novel metric suited to the downstream task of OGD. Experimentally, we find that SketchOGD tends to outperform current state-of-the-art variants of OGD given a fixed memory budget.\footnote{We provide code for SketchOGD and for reproducing our results at \url{https://github.com/azizanlab/sketchogd}.}
\end{abstract}

\section{Introduction}
Humans excel at learning continuously, acquiring new capabilities throughout their lifespan without repeatedly relearning previously mastered tasks. Moreover, knowledge gained in one context can often transfer to other tasks, and human learning remains relatively robust in the face of natural distribution shifts~\citep{barnett2002and,bremner2012multisensory}. In contrast, endowing deep neural networks with similar continual learning abilities remains a challenge, largely due to \textit{catastrophic forgetting} \citep{mccloskey}. Under the naïve approach of training networks on tasks sequentially, performance on earlier tasks can decline precipitously once new tasks are introduced.

One of the most important requirements for a continual learning algorithm intended to overcome catastrophic forgetting is that it be memory-efficient. Various approaches have been proposed to combat catastrophic forgetting (as briefly introduced in Section \ref{sec:related}). These approaches tend to involve storing information about past tasks to preserve trained knowledge while learning new tasks. Since continual or lifelong learning frameworks should consider learning new tasks over arbitrarily long time horizons with fixed memory budgets, memory cost is a chief consideration for practicality.

To advance memory-efficient continual learning, we center our attention on a conceptually straightforward yet powerful continual learning algorithm known as orthogonal gradient descent (OGD) \citep{Farajtabar2019OrthogonalGD}. Building on the properties of overparameterized neural networks \citep{li2018learning,azizan2018stochastic,azizan2019stochastic,allen2019convergence}, OGD preserves performance on previous tasks by storing model gradients for previous training examples and projecting later weight updates into the subspace orthogonal to those gradients. We choose to focus on OGD because it is simple, effective, and amenable to analysis. Furthermore, unlike other continual learning algorithms such as A-GEM \citep{Chaudhry2018EfficientLL}, it does not require re-computing forward passes on old tasks, meaning that there is no need to store raw datapoints. This reduces computation requirements and may be beneficial in a setting where privacy is important.
However, a key limitation of OGD is the memory cost of storing large numbers of gradients. For a model with $p$ parameters trained for $T$ iterations, the memory requirement is $\mathcal{O}(pT)$. For deep overparameterized networks, $p$ itself may be large, and the growth in the memory requirement with $T$ makes OGD ill-suited to learning over long time horizons.


To resolve the memory issues of OGD, this paper proposes a memory-efficient continual learning algorithm called \textit{sketched orthogonal gradient descent}---or SketchOGD. More specifically, we leverage \emph{sketching} algorithms \citep{Tropp} to compress the past model gradients required by OGD into a lower-dimensional representation known as a sketch.
Crucially, we can update the sketch efficiently (without needing to decompress and recompress gradients already encountered) such that 
SketchOGD can run online. In short, SketchOGD avoids storing the full history of past gradients by maintaining a sketch of a fixed size that may be updated cheaply and online.
To ensure that SketchOGD is a viable continual learning algorithm, we need to characterize how effective the sketching is in compressing past gradients. We tackle this question both theoretically and experimentally.
More specifically, \mbox{our contributions are as follows}:
\begin{enumerate}[leftmargin=0.62cm,itemsep=0pt,topsep=0pt]
    \item[\S~\ref{sec:algorithm}] We propose SketchOGD, a memory-efficient continual learning algorithm that leverages matrix sketching to address the memory problem of continual learning. In particular, we devise three variants, each employing a different sketching technique, suitable for different scenarios.
    \item[\S~\ref{sec:theory}] On the theoretical side, we provide formal guarantees on the approximation error of the three proposed variants of SketchOGD. These guarantees depend on the spectral structure of the matrix composed of past gradients and help to inform which sketching variant is best to use. 
    \item[\S~\ref{sec:experiment}] 
    We experimentally benchmark SketchOGD against existing memory-efficient OGD variants 
    and demonstrate that SketchOGD tends to outperform other OGD variants.
\end{enumerate}








\subsection{Related Work}
\label{sec:related}

\paragraph{Continual learning} Research on lifelong and continual learning focuses on solving the problem of catastrophic forgetting \citep{mccloskey, Ratcliff1990ConnectionistMO}. Some approaches are based on rehearsing old data/experiences whilst learning new tasks \citep{rehearsal,shin2017continual,rebuffi2017icarl,rolnick2019experience,van2020brain,pellegrini2020latent} or allocating neural resources incrementally for new tasks \citep{Rusu2016ProgressiveNN,valkov2018houdini,veniat2021efficient, ostapenko2021continual}, while others rely on regularizing training to preserve performance on earlier tasks \citep{Kirkpatrick2016OvercomingCF}. \textit{Orthogonal gradient descent} is one such regularization-based approach, which we discuss in Section \ref{sec:ogd}.
Other methods include elastic weight consolidation (EWC) \citep{Kirkpatrick2016OvercomingCF}, synaptic intelligence \citep{Zenke2017ContinualLT} and averaged gradient episodic memory (A-GEM) \citep{Chaudhry2018EfficientLL}.
\citet{PARISI201954} provide a more thorough review of recent developments.
A variant of OGD, known as ORFit, was recently proposed by \citet{ORFit} for the special setting of one-pass learning.

\paragraph{Matrix approximation} Many approaches have been proposed for approximating or compressing a matrix, including techniques based on sparsification and quantization \citep{achlioptas}, column selection \citep{selection}, 
dimensionality reduction \citep{halko} and approximation by submatrices \citep{submatrix}. One practical form of dimensionality reduction known as \textit{sketching} is introduced in Section \ref{sec:sketching}. \citet{halko} provide a brief survey of the matrix approximation literature.
Sketching has also been successfully employed for summarizing the curvature of neural networks for the purpose of out-of-distribution detection \citep{sharma2021sketching}, motivating its application to OGD.
\section{Background}


\subsection{Orthogonal Gradient Descent}
\label{sec:ogd}
As a continual learning algorithm, orthogonal gradient descent (OGD) attempts to preserve predictions on prior training examples as new training examples are encountered. To accomplish this, OGD stores the gradient of the model outputs for past training examples and projects future updates into the subspace orthogonal to these stored gradients. This locally prevents new updates from changing what the model outputs on previous tasks.

To be more precise, given a model $f$ with weights $w\,{\in}\, \R^p$, OGD stores the gradients $\nabla_w f(x_1,w), ..., \nabla_w f(x_n,w)$ for training examples $x_1,...,x_n$. This paper focuses on classification, and only the gradients of the ground-truth logit are stored which has been found to improve performance empirically
\cite{Farajtabar2019OrthogonalGD}. The $n$ gradient vectors are arranged into a matrix $G \in \R^{p \times n}$, and future weight updates are projected onto the orthogonal subspace $(\range G)^\perp$.
To perform the projection, one may form an orthonormal basis for $\range G$. Letting $\ortho G$ denote a matrix whose columns form such a basis, then the OGD weight update may be written as
\begin{equation*}
    w \gets w + (I - \ortho G(\ortho G)^\top) \Delta w.
\end{equation*}

\subsection{\rev{Memory Compression for OGD}}

A major limitation of OGD is the memory requirement of storing all $n$ past gradients, where the number of trained examples $n$ grows with the runtime of the algorithm. A simple solution is to only store a random subset of past gradients, in the hope that the corresponding submatrix of $G$ will accurately approximate the range of the full $G$ matrix. We call this method RandomOGD. 
Another solution called PCA-OGD has been proposed \citep{Doan2020ATA}, which stores only the top principal components of the gradients for each task. Similarly, ORFit \citep{ORFit} uses incremental PCA to derive a variant of OGD that is specialized for one-pass learning.

\rev{The core observation behind compressing the gradient matrix $G$ is that not all directions in $G$ are equally important. Figures 2 and 3 show that the singular values of the gradient matrices exhibit strong decay, so most of the ``energy" of $G$ is concentrated in a relatively low-dimensional subspace. This suggests that many directions in $\range G$ rarely contribute to preventing catastrophic forgetting across tasks. By focusing on the dominant directions, we can retain the part of the subspace that contributes most to OGD's core mechanism.

However, classical PCA used in PCA-OGD requires storing all gradients for a task to perform compression, which introduces a memory overhead that grows with task size and requires knowing task boundaries. Also, its memory to store the compressed results grows with each task, which makes it hard to
predict memory costs without knowing the number of tasks in advance.}

Incremental PCA~\citep{artac2002incremental, ozawa2008incremental, balsubramani2013fast} could be directly applied to OGD to compress the memory to a fixed size in an online manner. However, \rev{it typically involves repeated truncations, which can discard directions before their importance becomes apparent across tasks. Its} choice of which principal components to discard, and thus, the quality of compression, is sensitive to the order in which the training data is encountered. The next section introduces an alternative means of approximation that avoids these issues.

\subsection{Matrix Approximation via Sketching}
\label{sec:sketching}






To reduce the memory cost of OGD, we use \textit{sketching}, which approximates a matrix in lower dimensions.
\rev{Our sketching methods update a fixed-size summary online as each new gradient arrives, independent of the total number of tasks or samples---crucial in continual learning scenarios with long, unknown horizons. They also enjoy linearity properties that make the final sketch equivalent to sketching the full gradients at once. Moreover, sketching is simple to implement and comes with well-developed approximation guarantees that we utilize in our analysis in Section~\ref{sec:theory}.
}

\begin{definition}
\label{def:sketch}
A \emph{sketch} $(\Omega, \Psi, Y, W)$ of a matrix $A \in \R^{m\times n}$ is an approximation formed by first drawing two i.i.d.\ standard normal matrices $\Omega \in \R^{n\times k}$ and $\Psi \in \R^{l\times m}$ for $k\leq m$ and $l\leq m$ with $k\le l$, and then forming the products:
$
Y = A\Omega \in \R^{m\times k}, 
W = \Psi A \in \R^{l\times n}.
$
\end{definition}

A sketch can form an approximation $\widehat{A}\in\R^{m\times n}$ to the original matrix $A\in\R^{m\times n}$. For example:

\begin{informalprop}[Direct approximation]
\label{prop:low-rank-approx}
Given a sketch $(\Omega, \Psi, Y, W)$ of a matrix $A\in\R^{m\times n}$, and letting $Q=\ortho Y\in\R^{m\times \rev{\rank(Y)}}$, $\widehat{A} \defeq Q (\Psi Q)^\dagger W$ approximates $A$\rev{, where $(\cdot)^\dagger$ denotes the pseudoinverse}.
\end{informalprop}
Rigorous justification for this is given by \citet{Tropp}. The quality of the approximation improves with the size of the sketch set by $k$ and $l$. Especially, $k$ determines the rank of $\widehat{A}$.

For the purposes of OGD, we only need to know the range of the gradient matrix. So, we want to extract the range of $A$ from its sketch. While using $\range \widehat A$ would seem natural, care must be taken that $\widehat A$ does not include spurious directions in its range, \emph{i.e.}, $\range \widehat A \not\supset \range A$. In fact, we have:
\begin{equation*}
    \range \widehat{A} \subseteq \range Q = \range Y  \subseteq \range A.
\end{equation*}
So, $\range Q$ alone approximates $\range A$. When $A$ is symmetric, we can use another approximation:



\begin{informalprop}[Symmetric approximation]
\label{prop:symm-approx}
Given a sketch $(\Omega, \Psi, Y, W)$ of a symmetric matrix $A\in\R^{m\times m}$, let $Q=\ortho Y\in\R^{m\times \rev{\rank(Y)}}$ and define $X \defeq (\Psi Q)^\dagger W\in\R^{\rev{\rank(Y)}\times m}$. Then, $\widehat{A}_\mathrm{sym} \defeq \half(Q X + X^\top Q^\top)$ approximates $A$.
\end{informalprop}
Formal justification is again given by \citet{Tropp}. To approximate $\range A$, we can use the range of the concatenation $[Q\; X^\top]$ from this proposition as we observe that:
\begin{align*}
    \range \widehat A_\mathrm{sym} 
    & \subseteq \range Q \cup \range X^\top\\
    & \subseteq \range Y \cup \range W^\top\\
    & \subseteq \range A \cup \range A^\top = \range A,
\end{align*}
where the last statement follows since $A$ is symmetric.



\begin{algorithm}[t]
\caption{SketchOGD. The algorithm requires a model with weight vector $w$, a base optimizer such as SGD, as well as sketching subroutines from Sketching Method \ref{alg:method1}, \ref{alg:method2} or \ref{alg:method3}. }
   \label{alg:applied-sketch}
\begin{algorithmic}
   \STATE $\mathtt{InitializeSketch}()$
   \FOR{{\bfseries task in $\mathbf{\{1,2,...,T\}}$}}
   \IF {task $\ge 2$}
   \STATE $B \gets \mathtt{ExtractBasisFromSketch}()$
   \ELSE
   \STATE {$B\gets 0$}
   \ENDIF
   \FOR{{\bfseries each training iteration}}
   \STATE {$\Delta w\gets \mathtt{BaseOptimizerUpdate}()$}
   \STATE $w \gets w + (I-BB^\top)\Delta w$
   \ENDFOR
   \FOR{{\bfseries each of $\mathbf{s}$ random training points}}
   \STATE \COMMENT{Let $g$ denote the gradient of the correct logit}
    \STATE $\mathtt{UpdateSketch}(g)$
  \ENDFOR
  \ENDFOR
\end{algorithmic}
\end{algorithm}
\begin{algorithm}[t]
\makeatletter
\renewcommand*{\ALG@name}{Sketching Method}
\setcounter{algorithm}{0}
\makeatother
\caption{Sketch of gradient matrix $G$.}
   \label{alg:method1}
\begin{algorithmic}
\STATE {\bfseries def} $\mathtt{InitializeSketch}()$:
\bindent
\STATE Set $Y\in \mathbb{R}^{p \times k}$ to zero
\eindent
\STATE {\bfseries def} $\mathtt{UpdateSketch}(g)$:
\bindent
\STATE Draw $\omega\in \mathbb{R}^{k}$ i.i.d.\ standard normal
\STATE $Y \gets Y + g\omega^\top$
\eindent
\STATE {\bfseries def} $\mathtt{ExtractBasisFromSketch}()$:
\bindent
\RETURN $\ortho Y$
\eindent
\end{algorithmic}
\end{algorithm}
\begin{algorithm}[t]
\makeatletter
\renewcommand*{\ALG@name}{Sketching Method}
\makeatother
\caption{Sketch of product $GG^\top$.}
   \label{alg:method2}
\begin{algorithmic}
\STATE {\bfseries def} $\mathtt{InitializeSketch}()$:
\bindent
\STATE Set $Y \in \mathbb{R}^{p \times k}$ to zero
\STATE Draw $\Omega \in \mathbb{R}^{p \times k}$ i.i.d.\ standard normal
\eindent

\STATE {\bfseries def} $\mathtt{UpdateSketch}(g)$:
\bindent
\STATE $Y \gets Y + g(g^\top\Omega)$
\eindent

\STATE {\bfseries def} $\mathtt{ExtractBasisFromSketch}()$:
\bindent
\RETURN $\ortho Y$
\eindent
\end{algorithmic}
\end{algorithm}
\begin{algorithm}[t]
\makeatletter
\renewcommand*{\ALG@name}{Sketching Method}
\makeatother
\caption{Symmetric sketch of product $GG^\top$.}
   \label{alg:method3}
\begin{algorithmic}
\STATE {\bfseries def} $\mathtt{InitializeSketch}()$:
\bindent
\STATE Set $Y \in \mathbb{R}^{p \times k}$ and $W\in \mathbb{R}^{l \times p}$ to zero
\STATE Draw $\Omega \in \mathbb{R}^{p \times k}$ and $\Psi \in \mathbb{R}^{l \times p}$ i.i.d.\ standard normal
\eindent

\STATE {\bfseries def} $\mathtt{UpdateSketch}(g)$:
\bindent
\STATE $Y \gets Y + g(g^\top\Omega)$
\STATE $W \gets W + (\Psi g)g^\top$
\eindent

\STATE {\bfseries def} $\mathtt{ExtractBasisFromSketch}()$:
\bindent
\STATE $Q \gets \ortho Y$
\STATE $U,T \gets \mathtt{QR}(\Psi Q)$
\STATE $X \gets T^\dagger (U^\top W)$
\RETURN $\ortho[Q\; X^\top]$
\eindent
\end{algorithmic}
\end{algorithm}

\rev{We note that the magnitudes of the gradients in $G$ play an important role in forming the approximations. In the case of \emph{full-memory} OGD, i.e., when we have the luxury of storing all the directions of $G$, the magnitudes are redundant for preserving performance. However, under a fixed memory budget where storing all gradients is infeasible, the magnitudes provide crucial information for deciding which directions to retain. In particular, a larger magnitude of the gradient implies that a small model parameter update along that direction would alter the model output more significantly. Thus, the gradients with larger magnitudes become more important for mitigating overall catastrophic forgetting. Sketching methods naturally exploit this magnitude information by weighting directions accordingly.}
\vspace{-0.cm}\section{Sketched Orthogonal Gradient Descent}
\label{sec:algorithm}

This section applies the matrix sketching methods introduced in Section \ref{sec:sketching} to obtain three variants of \textit{sketched orthogonal gradient descent} (SketchOGD). We propose different variants as it is a priori unclear which performs best, but we will investigate this both by presenting theoretical bounds (Section~\ref{sec:derivation}) and by running experiments (Section~\ref{sec:experiment}). All variants adopt the following pattern:
\begin{enumerate}[
    leftmargin=.4cm,
    itemindent=0cm,
    itemsep=0cm,
    labelsep=.1cm,
    labelwidth=.6cm]
    \item Before training on a new task, use the maintained sketch to extract an orthonormal basis that approximately spans $\range G$, where $G$ is the matrix of gradients from previous tasks. 
    \item While training on the new task, project weight updates to be orthogonal to the orthonormal basis.
    \item After the training, calculate gradients of the model's correct logit and update the sketch.
\end{enumerate}



This procedure is described formally in Algorithm \ref{alg:applied-sketch}. Three variants are obtained by using either Sketching Method \ref{alg:method1}, \ref{alg:method2} or \ref{alg:method3} to supply the subroutines $\mathtt{InitializeSketch}$, $\mathtt{UpdateSketch}$ and $\mathtt{ExtractBasisFromSketch}$. The variants arise since there is freedom in both which matrix is sketched and what sketching technique is used. For instance, since $G$ and $GG^\top$ have the same range, either matrix may be sketched to approximate $\range G$. Furthermore, since $GG^\top$ is symmetric, Proposition \ref{prop:symm-approx} may be applied.
To elaborate, we propose three variants of SketchOGD:

\textbf{SketchOGD-1:} \textit{Algorithm \ref{alg:applied-sketch} + Sketching Method \ref{alg:method1}.}
This algorithm directly sketches the gradient matrix $G$ via Proposition \ref{prop:low-rank-approx}, and $\ortho Y$ for $Y = G \Omega$ forms the orthonormal basis. When a new gradient $g$ is received, we update $Y$ online instead of recomputing it from scratch since $Y^{(new)} = G^{(new)}\Omega^{(new)} = [G\; g] \begin{bsmallmatrix}\Omega\\ \omega^\top\end{bsmallmatrix} = Y + g \omega^\top$. Here, the random matrix $\Omega$ is incremented with a standard normal vector $\omega$ since the dimension of $G$ is changed. In terms of memory, SketchOGD-1 needs to store only $Y$ at a cost of $pk$.


\textbf{SketchOGD-2:} \textit{Algorithm \ref{alg:applied-sketch} + Sketching Method \ref{alg:method2}.} This algorithm sketches the product $GG^\top$ via Proposition \ref{prop:low-rank-approx} as in SketchOGD-1. To update the sketch variable $Y = GG^\top \Omega$ online when a new gradient $g$ is received, we leverage the decomposition $Y^{(new)} = G^{(new)}{G^{(new)}}^\top \Omega = \sum_{c\in G} cc^\top \Omega = Y + gg^\top \Omega$
where the summation is running over the columns of $G$. In terms of memory, SketchOGD-2 needs to store both $Y$ and $\Omega$ at a cost of $2pk$.


\textbf{SketchOGD-3:} \textit{Algorithm \ref{alg:applied-sketch} + Sketching Method \ref{alg:method3}.} This algorithm sketches the product $GG^\top$ via symmetric sketch (Proposition \ref{prop:symm-approx}). The sketch variables $Y$ and $W$ are updated online using the decomposition of $GG^\top$ as in SketchOGD-2. The computation of the pseudoinverse $(\Psi Q)^\dagger$ is expedited by first taking the orthogonal-triangular ($\mathtt{QR}$) decomposition. The full sketch needs to be stored, at a memory cost of $2p(l+k)$. Since $l\geq k$, this is at least four times that of SketchOGD-1. However, the basis returned by Sketching Method \ref{alg:method3} spans a higher-dimensional subspace.

Overall, SketchOGD-1 is at least twice as memory efficient as the other two methods, so we expect by default for it to do well under the same memory constraint. However, SketchOGD-2,3 may have a higher quality sketch by leveraging the symmetric structure of $GG^\top$.

\rev{As for computational complexities, SketchOGD-1,2 require $O(pk)$ for updating the sketch and $O(pk^2)$ for extracting the orthonormal basis from the sketch via QR decomposition. For SketchOGD-3, updating the sketch costs $O(p(l+k))$. Extracting the basis involves QR decompositions of $Y$ and $\Psi Q$, and subsequent matrix multiplications to compute $X$. This yields a total cost on the order of $O(p(l+k)k)$. Note that $k$ and $l$ are chosen to be much smaller than the number of parameters $p$, so these additional costs are modest compared to the total training time dominated by backpropagation.}







\section{Bounds on Sketching Accuracy}
\label{sec:theory}

This section theoretically analyzes the suitability of the sketching methods for use in SketchOGD.

\subsection{Constructing a Metric Suited to SketchOGD}
To assess the effectiveness of a sketching method, an error metric is needed. In continual learning, the goal is to preserve performance on all data points previously encountered. Therefore, in the context of SketchOGD, it makes sense to measure sketching accuracy as a sum of errors incurred on each past data point. We propose the following metric:
\begin{definition}[Error in a sketching method]\label{def:metric} Given a matrix $G$ to be approximated and \rev{an orthonormal basis} matrix $B$ returned by a sketching method, the reconstruction error is given by:
\begin{equation*}    \mathcal{E}_G(B)=\sum_{g\in G} \|g-\proj_{B}g\|_2^2 = \|\rev{(I-BB^\top)} G\|_F^2,
\end{equation*}
where the summation is running over the columns of G\rev{, and $\proj_{M}$ for a matrix $M$ denotes the projection onto $\range M$}. In SketchOGD, if the metric vanishes while $\range B \subseteq \range G$, then SketchOGD is equivalent to OGD.
\end{definition}

Let us distinguish Definition \ref{def:metric} from other possible error metrics. A common way to assess error in the sketching literature is to use a sketch to compute an approximation $\widehat{A}$ to a matrix $A$, and then to report the Frobenius reconstruction error $\|\widehat{A}-A\|_F$. The advantage of Definition \ref{def:metric} over this is that it decouples the matrix that is sketched from the matrix $G$ that we would like to reconstruct. In particular, the matrix $B$ may be obtained from a sketch of $GG^\top$ rather than $G$ itself.


Another possible metric is the Grassmann distance between subspaces \citep{grassmann}, which could be applied to measure the distance between the subspace $\range G$ and the subspace $\range B$, where $B$ is returned by the sketching method. The drawback to this approach is that it neglects how the columns $\{g\in G\}$ are distributed within $\range G$. If certain directions appear more frequently, then it is more important for OGD that these directions are well-modeled.

\subsection{Deriving Bounds}
\label{sec:derivation}

We now provide a novel theoretical analysis of sketching applied to continual learning, through bounding our metric. While \citet{Tropp} analyzed the Frobenius reconstruction error of sketching $\mathbb{E}\|A-\widehat{A}\|_F^2$, and \citet{Doan2020ATA} bounded catastrophic forgetting, to our knowledge, we are the first to theoretically analyze the disruption a compression method causes to OGD. 

Bounding our metric will depend on the sketching method involved. For the basic, first sketching method, we can adapt a result of Halko. However, for the second and third methods, which sketch $GG^\top$ instead of $G$, we state and prove original bounds.
In all three cases, our bounds will use some shared notation, including decomposing $G$ and splitting matrices along an index $\gamma$.

\begin{definition}[Split SVD] \label{def:thm-setup}
    
Define $U\in\R^{p\times p},\Sigma_G\in\R^{p\times n},V\in\R^{n\times \rev{n}}$ as the singular value decomposition (SVD) of $G$ such that $G = U \Sigma_G V$, where $n$ is the number of sketched gradients. Then, we can define $\Sigma \defeq \Sigma_G \Sigma_G^\top\in\R^{p\times p}$ such that $GG^\top = U \Sigma U^\top$. Assuming singular values sorted in decreasing magnitude, we define $\Sigma_1,\Sigma_2, U_1, U_2$ by splitting matrices at an index $\gamma \le k-2$: 
\begin{equation*}
    \Sigma = \begin{bmatrix}
 \Sigma_1 & 0\\
0 & \Sigma_2 
\end{bmatrix}, \qquad U = \begin{bmatrix}
 U_1 & U_2 
\end{bmatrix}.
\end{equation*}

\end{definition}

Theoretical bounds on our metric based on this split would imply their dependence on the singular value decay of the gradients.
In the basic sketching method where we directly sketch the matrix of gradients $G$, we can adapt a theorem of \cite{halko} as below. See the appendix for the proof. 
\begin{restatable}[Expected error in Sketching Method \ref{alg:method1}]{thm}{theoremone}
\label{thm:method-1}
Given that we create \rev{a sketch of $G$ with a parameter $k$ and obtain an orthonormal basis matrix $B$} in accordance with Sketching Method \ref{alg:method1}, and use the setup described in Definition \ref{def:thm-setup}, then the expected value of the metric $\mathcal{E}_G(\rev{B})$ has an upper bound:
\begin{equation*}
\mathbb{E}[\mathcal{E}_G(\rev{B})] \le \min_{\gamma\leq k-2} (1+\tfrac{\gamma}{k-\gamma-1})\cdot\Tr{\Sigma_2}.
\end{equation*}
\end{restatable}
\begin{remark}
Note that sharper decays push the optimal split location $\gamma$ away from $0$. For intuition, consider a flat spectral decay such that $\Sigma=\lambda I_{p\times p}$ for some $\lambda>0$ and the identity matrix $I_{p\times p}$. Then, $\gamma=0$ attains the minimum of the upper bound $p\lambda$. On the other hand, if $\Sigma=\Sigma_G\Sigma_G^T$ has linearly decaying spectral decay from $2\lambda$ to $0$, then $\gamma=2(k{-}1){-}p$ attains the minimum approximately at $\frac{4\lambda}{p} (k{-}1)(p{-}k{+}1)$ for large enough $k$. Analysis can be found in the appendix. 
\end{remark}

We now provide two novel theorems for Sketching Method \ref{alg:method2}, one upper bounding the metric deterministically and the other in expectation. See the appendix for the proofs. 

\begin{restatable}[Deterministic error in Sketching Method \ref{alg:method2}]{thm}{theoremdet}
\label{thm:deterministic-bound}
We define $G,\Sigma_1,\Sigma_2, U_1, U_2$ as in Definition \ref{def:thm-setup} for a sketch parameter $k$ and split position $\gamma\le k-2$. We create a sketch $(\Omega, \Psi, Y, W)$ of $A = GG^\top$ and \rev{obtain an orthonormal basis matrix $B$} in accordance with Sketching Method \ref{alg:method2}. We use $\Omega\in \mathbb{R}^{p\times k}$ to define: $\Omega_1 := U_1^\top \Omega$, $\Omega_2 := U_2^\top \Omega$.
Assume $\Omega_1\in \mathbb{R}^{\gamma \times k}$ is full-rank.
Then, if $\gamma$ is greater than the rank of $G$, then the metric $\mathcal{E}_G(\rev{B})=0$. Else,
\begin{equation*}
    \mathcal{E}_G(\rev{B}) \le \|\Sigma_2 \Omega_2 \Omega_1^\dagger \Sigma_1^{-1/2}\|_F^2 + \|\Sigma_2^{1/2}\|_F^2.
\end{equation*}
\end{restatable}

By evaluating the expected value of the bound in Theorem \ref{thm:deterministic-bound}, we get the following result.


\begin{restatable}[Exptected error in Sketching Method \ref{alg:method2}]{thm}{theoremexp}
\label{thm:expectation-bound}
We define $G,\Sigma_1,\Sigma_2, U_1$ as in Definition \ref{def:thm-setup} for a sketch parameter $k$ and arbitrary split position $\gamma\le k-2$. We create a sketch $(\Omega, \Psi, Y, W)$ of $A = GG^\top$ \rev{and obtain an orthonormal basis matrix $B$ in accordance with Sketching Method \ref{alg:method2}.} We assume $\Omega_1 := U_1^\top \Omega$ has full rank.
Then, if $\gamma$ is greater than the rank of $G$, then the metric $\mathbb{E}[\mathcal{E}_G(\rev{B})]=0$. Else,
\begin{equation*}
    \mathbb{E}[\mathcal{E}_G(\rev{B})]\le \min_{\gamma\le k-2} \tfrac{\gamma}{k-\gamma-1}\Tr(\Sigma_2^{2}) \Tr(\Sigma_1^{-1}) + \Tr(\Sigma_2).
\end{equation*}
\end{restatable}

\begin{remark} \label{rem:tighter_bound}
Note that this is a tighter bound than Theorem \ref{thm:method-1} when $\Tr(\Sigma_2^{2}) \Tr(\Sigma_1^{-1}) \le \Tr(\Sigma_2)$. The smaller the lower singular values are, and the larger the higher singular values are, the more advantageous this bound becomes. One illustrative example is a step-decaying spectral decay such that $\Sigma_1=100 I_{10\times 10}$ and $\Sigma_2=2 I_{100\times 100}$. Then, $\Tr(\Sigma_2^{2}) \Tr(\Sigma_1^{-1}) = 40 \le 200 = \Tr(\Sigma_2)$. 
\end{remark}


Sketching Method \ref{alg:method3} is a modification of Sketching Method \ref{alg:method2} to take advantage of the symmetric structure of $GG^\top$
\rev{via the symmetric sketch in Proposition~\ref{prop:symm-approx}. Sketching Method \ref{alg:method3} returns a basis that spans a larger subspace than that from Sketching Method \ref{alg:method2}. Any projection onto this larger subspace can only reduce the reconstruction error compared to Sketching Method \ref{alg:method2}. Thus,}
the metric for Sketching Method \ref{alg:method3} is always better than for Sketching Method \ref{alg:method2} as below, so we can transfer our previous results. See the appendix for the proof.

\begin{restatable}[Error in Sketching Method \ref{alg:method3}]{thm}{theoremthree}
\label{thm:symmetric} When sketching $A=GG^\top$ \rev{to obtain orthonormal bases $B$ and $B_\mathrm{sym}$} according to Sketching Methods \ref{alg:method2} and \ref{alg:method3}, respectively, the metrics obey: \begin{equation*}
    \mathcal{E}_G(\rev{B_\mathrm{sym}}) \le \mathcal{E}_G(\rev{B}).
\end{equation*}
\end{restatable}

Consequently, the bounds described in Theorems \ref{thm:deterministic-bound} and \ref{thm:expectation-bound} are also upper bounds for the error metric $\mathcal{E}_G(\rev{B_\mathrm{sym}})$ and its expected value $\mathbb{E}[\mathcal{E}_G(\rev{B_\mathrm{sym}})]$, respectively.

\begin{remark}
Overall, these bounds suggest that the performance of different sketching methods may vary with the task-dependent structure of spectral decay. \rev{More sharply decaying spectrum of $GG^\top$ makes the bound for Sketching Methods~\ref{alg:method2} and \ref{alg:method3} in Theorem~\ref{thm:expectation-bound} tighter than that for Sketching Method~\ref{alg:method1} in Theorem~\ref{thm:method-1}. Conversely, when the spectrum decays slowly, Sketching Method~\ref{alg:method1} can be more favorable. Among Sketching Methods~\ref{alg:method2} and \ref{alg:method3}, the metric for Sketching Method~\ref{alg:method3} is always better than Sketching Method~\ref{alg:method2} as noted in Theorem~\ref{thm:symmetric}.}
\end{remark}

\rev{
\section{Analysis of Convergence and Catastrophic Forgetting}

In this section, we theoretically characterize the limiting parameter attained by SketchOGD on each task and analyze the catastrophic forgetting induced by the successive parameter updates along tasks.

Under the Neural Tangent Kernel (NTK) regime, \citet{bennani2020generalisation} and \citet{Doan2020ATA} interpret OGD as a recursive kernel regression and analyze its forgetting behavior. We adopt a similar approach and establish an explicit connection between the sketching reconstruction error $\mathcal{E}_G(B)$ and the catastrophic forgetting.

The key idea of the NTK regime is that, for a neural network with large enough width, the model is well-approximated by its first-order approximation around the initialization $w_0\in\mathbb{R}^p$:
\begin{equation}\label{eq:ntk}
    f(x,w) \approx f(x,w_0) + \nabla_w f(x,w_0)^\top (w-w_0),
\end{equation}
where the feature map $\nabla_w f(x,w)$ remains constant throughout training~\citep{JacotNTK, LeeNTK}. This yields a kernel regression viewpoint of learning and enables closed-form characterizations of the solutions reached by SketchOGD.

For each task $\tau\in\{1,2,\dots,T\}$, we train on the dataset $\{(x_{\tau,i}, y_{\tau,i})\}_{i=1}^{n_\tau}$, arranged as $\mathtt{X}_\tau\in\mathbb{R}^{d\times n_\tau}$, and $\mathtt{Y}_\tau\in\mathbb{R}^{n_\tau}$, with regularized squared loss function:
\begin{equation}\label{eq:loss}
    \mathcal{L}_\tau(w)=\frac{1}{n_\tau}\|f(\mathtt{X}_\tau,w)-\mathtt{Y}_\tau\|_2^2+\lambda\|w-w_{\tau-1}^*\|_2^2,
\end{equation}
where $f(\mathtt{X}_\tau,w)\in\mathbb{R}^{n_\tau}$ is an aggregated function evaluation on each input of $\mathtt{X}_\tau$.
Let $G_\tau\in\mathbb{R}^{p\times n_\tau}$ denote the matrix whose columns are the gradients for the $n_\tau$ training data. Prior to training task $\tau$, a sketching method constructs
$B_\tau$ based on the gradients of prior tasks $G_{1:\tau-1}$.
Denoting the orthogonal projection as $T_\tau:=I_p-B_\tau B_\tau^\top$, each $j$-th update step of SketchOGD can be represented as:
\begin{equation}\label{eq:sketchogd_step}
    w_\tau^{(j+1)} = w_\tau^{(j)} - \eta_\tau T_\tau \nabla_w \mathcal{L}_\tau(w_\tau^{(j)}),
\end{equation}
for a fixed step-size $\eta_\tau>0$.

Under this setting, we characterize the convergence behavior of SketchOGD as below. See the appendix for the proof.
\begin{restatable}[Convergence analysis]{thm}{theoremconv}
\label{thm:conv}
Suppose the NTK assumption~\eqref{eq:ntk} holds and the step-size $\eta_\tau$ is sufficiently small so that $0<\eta_\tau<\dfrac{1}{\lambda_{max}\big(T_\tau(\frac{1}{n_\tau}G_\tau G_\tau^\top+\lambda I_p)\big)}$ for all $\tau\in\{1,\dots,T\}$. Then, for each task $\tau$, the parameter updates $\{w_\tau^{(j)}\}_{j\geq0}$ of SketchOGD in~\eqref{eq:sketchogd_step} converge linearly to a limit solution $w_\tau^*$ such that
\begin{equation}\label{eq:param_conv}
    w_\tau^*-w_{\tau-1}^* = T_\tau G_\tau(G_\tau^\top T_\tau G_\tau+n_\tau\lambda I_{n_\tau})^{-1} \tilde{\mathtt{Y}}_\tau,
\end{equation}
where $\tilde{\mathtt{Y}}_\tau:=\mathtt{Y}_\tau-f(\mathtt{X}_\tau,w_{\tau-1}^*)$ and $w_0^*=w_0$.
\end{restatable}

The recursive characterization of the limiting parameters in~\eqref{eq:param_conv} enables a direct analysis of catastrophic forgetting. Similarly to the definition from~\citet{Doan2020ATA}, we define catastrophic forgetting on the training data as follows.
\begin{definition}[Catastrophic forgetting]
    Let $\tau_S$ and $\tau_T$ ($\tau_S<\tau_T$) be the source and target tasks, respectively. The catastrophic forgetting of task $\tau_S$ after training on all subsequent tasks up to task $\tau_T$ is
    \begin{equation*}
        CF_{\tau_S\rightarrow\tau_T} = \sum_{i=1}^{n_{\tau_S}} \big(f(x_{\tau_S,i},w_{\tau_T}^*) - f(x_{\tau_S,i},w_{\tau_S}^*)\big)^2,
    \end{equation*}
    where $w_{\tau_S}^*, w_{\tau_T}^*$ are the limiting parameters from Theorem~\ref{thm:conv}.
\end{definition}

Then, with the NTK approximation, the catastrophic forgetting admits the equivalent form:
\begin{align*}
    CF_{\tau_S\rightarrow\tau_T} &= \|G_{\tau_S}^\top(w_{\tau_T}^* - w_{\tau_S}^*)\|^2\\
    &= \big\|\sum_{\tau=\tau_S+1}^{\tau_T} G_{\tau_S}^\top(w_{\tau}^* - w_{\tau-1}^*)\big\|^2.
\end{align*}
Combining this identity with the recursive characterization of $w_{\tau}^*$ in~\eqref{eq:param_conv} yields the following results.
\begin{restatable}[Catastrophic forgetting of SketchOGD]{thm}{theoremcf}
Suppose the conditions of Theorem~\ref{thm:conv} hold. Then, for the source task $\tau_S$ and target task $\tau_T$ ($\tau_S<\tau_T$),
\begin{equation*}
    CF_{\tau_S\rightarrow\tau_T}
    = \big\|\sum_{\tau=\tau_S+1}^{\tau_T} G_{\tau_S}^\top T_\tau G_\tau(G_\tau^\top T_\tau G_\tau+n_\tau\lambda I)^{-1} \tilde{\mathtt{Y}}_\tau\big\|^2.
\end{equation*}
\end{restatable}

\begin{remark}
    For each task $\tau>\tau_S$, the reconstruction error defined in Definition~\ref{def:metric} upper bounds the magnitude of the cross-term $G_{\tau_S}^\top T_\tau$ since the error metric after task $\tau-1$ is
    \begin{equation*}
        \mathcal{E}_{G_{1:\tau-1}}(B_\tau)=\|T_\tau G_{1:\tau-1}\|_F^2 \geq \|G_{\tau_S}^\top T_\tau\|_F^2.
    \end{equation*}
    A small reconstruction error implies that $G_{\tau_S}$ is nearly contained in the sketched subspace, thereby leading to smaller catastrophic forgetting.
\end{remark}

}
\section{Experiments}
\label{sec:experiment}

In this section, we demonstrate the practical performance of SketchOGD on four continual learning benchmarks: Rotated MNIST, Permuted MNIST, Split MNIST, and Split CIFAR.

\paragraph{Setup}
The benchmarks are created using $1,000$ datapoints for each of the $10$ classes of MNIST, and $600$ datapoints for each of the $20$ classes of CIFAR-100. Rotated MNIST consists of $10$ tasks, created by copying and rotating the dataset by multiples of $5$ degrees. Similarly, Permuted MNIST forms $10$ tasks by creating fixed but random permutations of pixels for each task.
Split MNIST forms $5$ tasks by splitting the dataset into pairs of digits and treating each pair as a binary $0,1$ classification task. Similarly, Split CIFAR has $20$ tasks of $5$ classes each. Each task is trained for $30$ epochs.

\rev{Our primary goal is to compare methods at fixed memory budgets, which is the typical design constraint in continual learning scenarios.} For the setup of our proposed SketchOGD methods, the memory constraints allow $k=1200$ for Method 1, $k=600$ for Method 2, and $k=299, l=301$ for Method 3. We set $k\approx l$, which is empirically found to be more memory-efficient. Effectively, SketchOGD-1 forms a rank-$1200$ approximation, while SketchOGD-2 and SketchOGD-3 form around rank-$600$ approximations.

\paragraph{Network Architecture}
For MNIST experiments, we use a three-layer fully connected neural network that has $100$ hidden units in the first two layers with ReLU activations and $10$ output logits, totaling $p=113,610$ parameters. For Split CIFAR, we use a LeNet with a hidden-layer size of $500$ for a total of $1,601,500$ parameters. Unlike previous works \citep{Farajtabar2019OrthogonalGD,Doan2020ATA} that create unique heads for each task in Split MNIST, we keep the model fixed to make Split MNIST a harder problem. However, we keep the multi-head architecture for Split CIFAR.

\paragraph{Baselines}

We compare the SketchOGD methods with other continual learning methods under the same memory constraint. They include naive SGD for context, a variant of OGD which we call ``RandomOGD'' that randomly samples gradients to obey the memory constraint, PCA-OGD proposed by \cite{Doan2020ATA} that compresses each task's gradients using PCA, and OGD without memory constraint.
We expect naive SGD to perform drastically worse than all other methods because of catastrophic forgetting. RandomOGD should serve as a lower bound, as random sampling is the simplest form of compression. Unconstrained OGD serves as the theoretical maximum if compression methods were lossless, and PCA-OGD is the current state-of-the-art memory-efficient OGD variant.

It is important to note that PCA-OGD is not an online algorithm within a task. It requires temporarily storing all the gradients for a task at once into a big matrix. Then, it performs singular value decomposition and stores the top left singular vectors as a compression. If a task is large enough, storing all the gradients could easily break the memory constraint. Therefore we run two different scenarios; first, a practical one where PCA-OGD accounts for this fault so that it cannot summarize the full dataset while SketchOGD can, taking advantage of its online nature; then, a toy scenario where PCA-OGD and SketchOGD are set to compress the same number of gradients regardless of the memory issue of PCA-OGD to compare their compression quality from the same ingredients.

\subsection{Experiment 1: All Methods under the Same Memory Constraint}

\begin{table}
\centering
\caption{Test accuracy after training on four different continual learning benchmarks, under a fixed memory constraint. A-GEM is shown separately, representing a commonly used algorithm that allows storing previous datapoints across tasks. The mean and standard deviation over 5 runs for each algorithm-benchmark pair are indicated, and exceptional results are bolded. The SketchOGD methods generally outperform the other variants of OGD, while being comparable to each other.}
\setlength{\tabcolsep}{2pt}
\resizebox{\columnwidth}{!}{%
\begin{tabular}{l c c c c}
\toprule
\textbf{Algorithm} & \textbf{Rotated} & \textbf{Permuted} & \textbf{Split MNIST} & \textbf{Split CIFAR}\\
\midrule
SketchOGD-1 &  $0.865\pm 0.004$ & $\textbf{0.832}\pm 0.004$ & ${0.679}\pm 0.004$ & ${0.676}\pm 0.002$\\
SketchOGD-2 & $\textbf{0.876} \pm 0.004$ & $0.826 \pm 0.004$ & $0.675 \pm 0.007$ & $0.672 \pm 0.002$\\
SketchOGD-3 & $\textbf{0.875}\pm 0.006$ & $0.825\pm 0.002$ & ${0.676}\pm 0.006$ & $0.673\pm 0.002$\\
PCA-OGD & $0.849\pm 0.004$ & $0.821\pm 0.004$ & $0.673 \pm 0.001$ & ${0.674} \pm 0.002$\\
RandomOGD & $0.847\pm 0.007$ & $0.819\pm 0.002$ & $0.667 \pm 0.004$ & $0.672 \pm 0.002$\\
SGD & $0.792\pm 0.009$ & $0.734\pm 0.019$ & $0.604 \pm 0.003$ & $0.649 \pm 0.002$\\
\midrule
A-GEM & $0.902 \pm 0.002$ & $0.887 \pm 0.001$ & $0.905 \pm 0.001$ & $0.379 \pm 0.024$\\
\bottomrule
\end{tabular}
}
\label{table:unlimited-sketch}
\end{table}

\begin{figure}[t]
  \centering
  \includegraphics[width=.9\columnwidth]{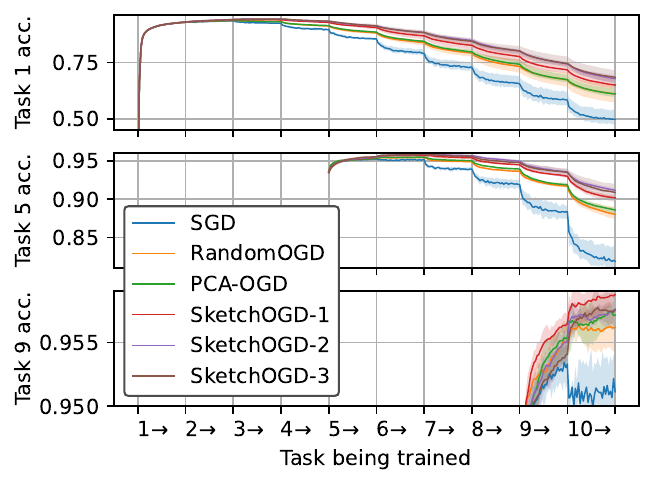}
  \caption{Accuracy of various continual learning methods on tasks 1, 5, and 9 of the Rotated MNIST benchmark. Each plot shows the initial learning of a task and the subsequent slow forgetting as the model trains on new tasks. As shown, SketchOGD-2,3 perform the best in this scenario.}
  \label{fig:rotated_graph}
\end{figure}
\begin{figure}[t]
  \centering
  \includegraphics[width=0.6\columnwidth]{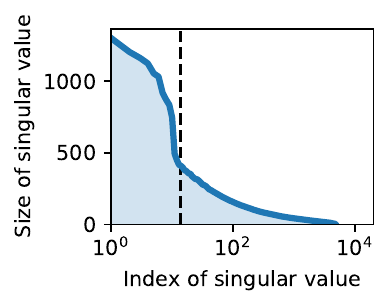}
  \caption{Plotted singular values of $4800$ gradient vectors extracted from training OGD on the Rotated MNIST benchmark. The x-axis is placed on a log scale for easier viewing, as the decay is too fast to make sense of otherwise. The stable rank of the matrix is shown as a dotted black line, which can be viewed as an effective rank. As shown, even taking a small number of the top singular vectors ($\approx 300$) would encode the majority of the information.}
  \label{fig:spectral_decay}
\end{figure}




Table \ref{table:unlimited-sketch} shows the first scenario's results. As sketching more gradients does not change memory costs in SketchOGD, we allow SketchOGD to sketch the full dataset, namely $100,000, 10,000$ and $60,000$ gradients, respectively, for Rotated/Permuted MNIST, Split-MNIST, and Split-CIFAR, down to $1200$ vectors. Meanwhile, as PCA-OGD takes extra memory space to run, it can only compress $2000, 1500$, and $4000$ gradients down to $1000, 900$ and $1000$ vectors, respectively, for the corresponding benchmarks under the same memory constraint. This experiment shows that the SketchOGD methods are comparable to each other, while outperforming PCA-OGD significantly. In addition, the popular continual learning algorithm A-GEM is included for reference, performing well for MNIST benchmarks but extremely poorly for Split CIFAR. Unlike the variants of OGD, A-GEM directly stores datapoints across tasks, which might raise privacy concerns.
Figure \ref{fig:rotated_graph} graphs the test accuracy of tasks $1,5,9$ of the Rotated MNIST benchmark along the training. RandomOGD acts as a simple lower bound for the OGD variants, and the SketchOGD methods outperform PCA-OGD.

\subsection{Experiment 2: Compressing the Same Number of Gradients}

\begin{table}
\centering
\caption{Performance comparison in a toy scenario where each method compresses the same number of gradients; this artificially restricts the capability of SketchOGD to compress arbitrarily many gradients without memory overhead while relatively boosting PCA-OGD. Unlike other methods, OGD is allowed to store all gradients without compression to form an upper bound of lossless compression. The mean and standard deviation of test accuracy over 5 runs for each algorithm-benchmark pair are indicated, and exceptional results are bolded. The SketchOGD methods show comparable performance to PCA-OGD despite the restriction.}
\setlength{\tabcolsep}{2pt}
\resizebox{\columnwidth}{!}{%
\begin{tabular}{l c c c c}
\toprule
\textbf{Algorithm} & \textbf{Rotated} & \textbf{Permuted} & \textbf{Split MNIST} & \textbf{Split CIFAR}\\
\midrule
SketchOGD-1 &  ${0.863}\pm 0.005$ & $\textbf{0.830}\pm 0.002$ & $\textbf{0.684}\pm 0.007$ & ${0.674}\pm 0.003$\\
SketchOGD-2 & ${0.860}\pm 0.006$ & $0.818\pm 0.005$ & $0.673\pm 0.006$ & ${0.675}\pm 0.002$\\
SketchOGD-3 & ${0.862}\pm 0.007$ & $0.812\pm 0.003$ & $0.677\pm 0.003$ & ${0.675}\pm 0.001$\\
PCA-OGD & ${0.864}\pm 0.004$ & $\textbf{0.830}\pm 0.002$ & $0.678 \pm 0.006$ & ${0.675}\pm 0.002$\\
RandomOGD & $0.847\pm 0.007$ & $0.819\pm 0.002$ & $0.667 \pm 0.004$ & $0.672 \pm 0.002$\\
\midrule
OGD & $0.888\pm 0.006$ & $0.838\pm 0.004$ & $0.714 \pm 0.002$ & $0.675\pm 0.006$\\
\bottomrule
\end{tabular}%
}
\label{table:accuracies}
\end{table}

Next, we restrict SketchOGD methods to only being fed in the same number of gradients to compress as PCA-OGD. This is a severe disadvantage as PCA-OGD has a memory overhead to store a full matrix of gradients pre-compression that SketchOGD does not need. Table~\ref{table:accuracies} shows the results where each method compresses $4800$ gradients down to $1200$ vectors, except for OGD, which stores all gradients without compression. In this experiment, we aim to understand how well sketching can compress data in comparison to PCA, as well as the effects of sketching different amounts of gradients on SketchOGD's results. The sketching methods on balance outperformed RandomOGD and naive SGD, while underperforming the upper bound OGD. Compared to the results in Table~\ref{table:unlimited-sketch}, while SketchOGD-2,3 drop in performance, SketchOGD-1 is not significantly affected by the reduced number of sketched gradients. Further, even SketchOGD-2,3 are still comparable to PCA-OGD on most benchmarks.

\section{Discussion}
\vspace{-.2cm}
In this section, we discuss the practical performance of the proposed SketchOGD methods for memory-efficient continual learning.
Comparing them, no one dominates the others, while all tend to outperform PCA-OGD as seen in Table \ref{table:unlimited-sketch}.
SketchOGD-2 and SketchOGD-3 exhibit similar behavior, indicating that the doubled memory cost of SketchOGD-3 is mitigated by its use of symmetric structure. Meanwhile, SketchOGD-2 outperforms SketchOGD-1 in Rotated MNIST despite its higher memory cost. This aligns with the tighter theoretical error bounds for SketchOGD-2,3 under sharper spectral gradient decay (Figure \ref{fig:spectral_decay}), as noted in Remark~\ref{rem:tighter_bound}. Conversely, slower spectral decay in Permuted MNIST (Figure \ref{fig:singular-values-all} in the appendix) benefits SketchOGD-1. One may expect the slower decay a priori as the random permutations would interrupt transferring learned features to the following tasks more than other benchmarks, resulting in less similar gradients between tasks. In this manner, intuition about tasks can inform method selection. 


Another intuition for the behavior of SketchOGD-2,3 is that sketching $GG^\top$ emphasizes larger gradients due to quadratic scaling of their magnitudes compared to simply $G$. This property is particularly effective if there are few key directions (equivalently, a sharper spectral decay).
In addition, Table \ref{table:accuracies} shows that SketchOGD-2,3 benefit more from sketching more gradients, while SketchOGD-1 improves less. This trend could be explained by the theoretical bounds, as Theorem \ref{thm:expectation-bound} is quadratic in the smallest singular values while Theorem \ref{thm:method-1} is linear and therefore is hurt more by additional small singular values.

Lastly, we highlight some practical advantages of SketchOGD over PCA-OGD.
First, SketchOGD has a fixed memory usage regardless of the number of tasks, whereas PCA-OGD's memory grows with each task. This makes it hard to predict memory costs and choose the right hyperparameters for PCA-OGD without knowing the number of tasks in advance. In addition, PCA-OGD incurs overhead from storing all gradients for a task before compression, limiting scalability for large task datasets.
Secondly, SketchOGD’s incremental updates summarize gradients across tasks, leveraging shared directions that PCA-OGD’s independent task processing may miss. Different tasks likely share some directions in their gradients, especially if subsequent tasks are similar or connected. Then, those shared directions can accumulate in importance and be picked up by the sketch, while PCA on individual tasks might discard them without knowing their shared importance across tasks.

\paragraph{Future Work}
Our results motivate several directions of future inquiry. 
First, Theorem \ref{thm:symmetric} leaves room for a more rigorous analysis of the benefits of leveraging the symmetric structure of $GG^\top$ in SketchOGD-3. Also, the difference between the bounds in \ref{thm:method-1} and \ref{thm:expectation-bound} involve replacing one instance of $\Tr (\Sigma_2)$ with $\Tr (\Sigma_2^2) \Tr (\Sigma_1^{-1})$; we would like to investigate the way different types of singular value decay would affect these two bounds differently.
Another area of importance is one of our potential explanations for SketchOGD-2,3's performance, namely the quadratically scaled magnitudes of gradients in $GG^\top$. We may test whether scaling the gradients by powers of their magnitudes before compression may influence results for the better.
In addition, given SketchOGD’s incremental nature, developing an incremental PCA-OGD variant for comparison is another promising direction. However, any incremental forms of PCA won't have the linear properties of sketching that make it equivalent to sketching the full matrix all at once; a modified PCA will likely still involve repeated truncation, discarding important directions before they have time to accumulate across tasks.
\rev{Another interesting direction is to analyze formal privacy guarantees~\citep{zhu2021rgap, wang2022protect}, given that one of SketchOGD's primary advantages is its ability to avoid storing raw data, which is a key source of privacy concerns. In addition, extending the empirical study to larger-scale datasets (e.g., CORe50, ImageNet) and modern architectures (e.g., ResNets, Transformers) represents an important avenue for future research.}
\vspace{-.2cm}
\section{Conclusion}
\vspace{-.1cm}
Memory efficiency is crucial for continual learning over long horizons. We proposed SketchOGD, an efficient algorithm for systematically resolving the growing memory problem of orthogonal gradient descent (OGD) in continual learning. Our sketching approach is easy to implement; it updates a compact summary of gradients online without relying on boundaries between the tasks, requiring constant memory irrespective of the number of gradients sketched. This makes it suitable for scenarios with uncertain futures. We provided a novel performance metric for the sketch, relevant to the downstream task of OGD, along with formal guarantees through deterministic and in-expectation bounds. Empirical results demonstrate that SketchOGD outperforms existing memory-efficient OGD variants, making it a robust choice for continual learning with fixed memory constraints.

\section*{Acknowledgements}

This work was supported in part by the MIT-IBM Watson AI Lab, MathWorks, the MIT-Amazon Science Hub, and the MIT-Google Program for Computing Innovation. The authors acknowledge the MIT SuperCloud and Lincoln Laboratory Supercomputing Center \cite{reuther2018interactive} for providing computing resources that have contributed to the research results reported within this paper.

\bibliography{refs}
\bibliographystyle{unsrtnat}

\appendix
\section{Appendix}
\label{sec:appendix}

\subsection{Spectral Decay}
\begin{figure}[h!]
    \centering
    \includegraphics[width=\columnwidth]{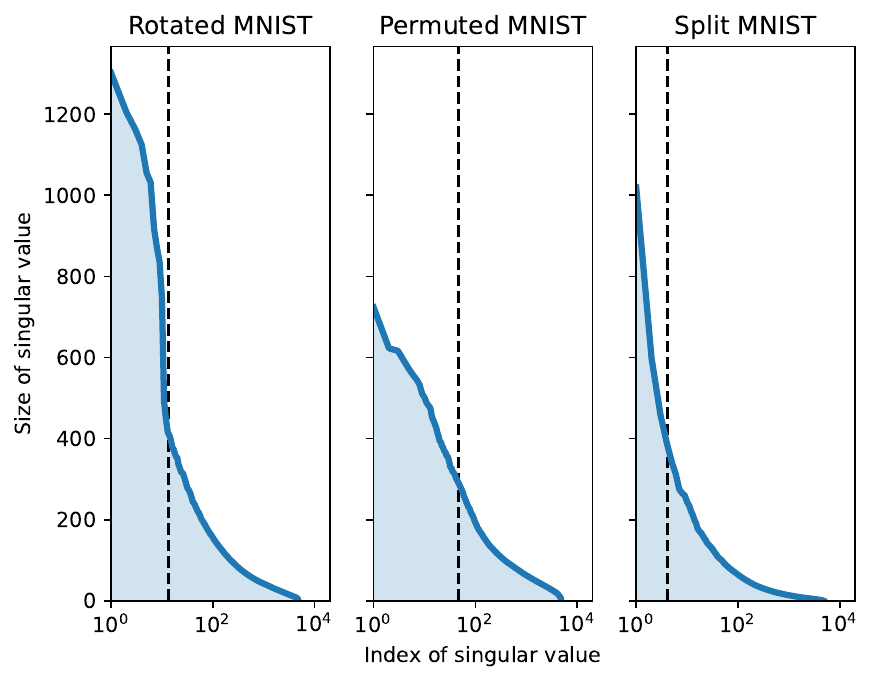}
    \caption{Plotted singular values of $4800$ gradient vectors extracted from training OGD on the Rotated MNIST, Permuted MNIST, and Split MNIST classification benchmarks. The x-axis is on a log scale for easier viewing, as the decay is too fast to make sense of otherwise. The stable rank of the matrix is shown as a dotted black line, which can be viewed as an effective rank. Permuted MNIST shows relatively slow spectral decay compared to the other two benchmarks. However, the decay is still sharp enough so that even taking a small number of the top singular vectors ($\approx 300$) would encode the majority of the information for all three benchmarks.}
    \label{fig:singular-values-all}
\end{figure}

\subsubsection{Effects of Spectral Decay on Bounds}

Recall that the theoretical upper bound on the expected error metric for SketchOGD-1 is $\min_{\gamma\leq k-2} (1+\tfrac{\gamma}{k-\gamma-1})\cdot\Tr{\Sigma_2}$ where $\Sigma_2$ is the lower singular value matrix of $\Sigma=\Sigma_G \Sigma_G^T$ split at the index $\gamma\leq k-2$, and $\Sigma_G$ is the singular value matrix of the gradient matrix $G$. Then, sharper decays push the optimal split location $\gamma$ away from 0. For intuition, we compare two different spectral decays.

First, consider a flat spectral decay such that $\Sigma=\lambda I_{p\times p}$ for some $\lambda>0$ and the $p\times p$ identity matrix $I_{p\times p}$ (Here, we consider a full-rank $\Sigma$ for simplicity, but it has rank $n$ when the number of the encountered datapoints $n$ is less than the number of parameters $p$). Then, we can solve for the optimal split location $\gamma^*$ that achieves the tightest upper bound:
\begin{align*}
    \gamma^* &= \argmin_{\gamma \le k-2} (1+\tfrac{\gamma}{k-\gamma-1})\cdot\Tr{\Sigma_2}\\
    &= \argmin_{\gamma \le k-2} \tfrac{k-1}{k-\gamma-1}\cdot\lambda(p-\gamma)
    = \argmin_{\gamma \le k-2} \tfrac{p-\gamma}{k-\gamma-1}
    = 0.
\end{align*}

Plugging $\gamma^*$ in, we get the upper bound $(1+\tfrac{\gamma^*}{k-\gamma^*-1})\cdot\Tr{\Sigma_2} = p\lambda$. Intuitively, a flat decay pushes the optimal split location $\gamma$ to $0$, and scaling the singular values by a constant does not affect the optimal split location.

In contrast, let's consider $\Sigma$ of which the diagonal values linearly decay from $2\lambda$ to $0$ over $p$ values. Then, the optimal split location can change:
\begin{align*}
    \gamma^* &= \argmin_{\gamma \le k-2} (1+\tfrac{\gamma}{k-\gamma-1})\cdot\Tr{\Sigma_2}\\
    &\approx \argmin_{\gamma \le k-2} \tfrac{k-1}{k-\gamma-1}\cdot\tfrac{2\lambda}{p}(p-\gamma)^2/2
    = \argmin_{\gamma \le k-2} \tfrac{(p-\gamma)^2}{k-\gamma-1},
\end{align*}
where the approximation for calculating $\Tr{\Sigma_2}$ comes from the simplification $1+2+\ldots+(p-\gamma) = (p-\gamma)^2/2$.

To calculate when the minimum is achieved, we can set the derivative with respect to $\gamma$ to be $0$:
\begin{align*}
    0 = \tfrac{d}{d\gamma}\tfrac{(p-\gamma)^2}{k-\gamma-1}|_{\gamma=\gamma^*}
    &= (p-\gamma^*)^2 - 2(p-\gamma^*)(k-1-\gamma^*)\\
    &= (p-\gamma^*) - 2(k-1-\gamma^*),
\end{align*}
which results in $\gamma^* = 2k-2-p$. As long as $2k-2-p\ge 0$, then the optimal split location becomes $\gamma^* = 2k-2-p$. Plugging in, the upper bound computes to $\frac{4\lambda}{p} (k{-}1)(p{-}k{+}1)$. Note that this bound is smaller than in the constant singular values case, even though the sum of the singular values is the same. Similarly, the faster the decay becomes, the more $\gamma^*$ is pushed away from $0$, and the tighter the theoretical upper bound becomes.

\subsection{Proofs}
\theoremone*

\begin{proof}
Expanding our metric results in:
\[\mathbb{E}[\mathcal{E}_G(B)] = \mathbb{E}\|(G-BB^\top G)\|_F^2. \]
Then, we can apply Theorem 10.5 of \citet{halko}:
\begin{align}
\mathbb{E}[\mathcal{E}_G(B)] \le \min_{\gamma< k-1} (1+\tfrac{\gamma}{k-\gamma-1})\cdot\tau^2_{\gamma+1}G,
\end{align}
where $\tau^2_{\gamma+1}G$ is the sum of the squares of the singular values of $G$, skipping the $\gamma$ largest singular values in magnitude. Therefore, $\tau^2_{\gamma+1}G = \Tr{\Sigma_2}$, finishing.
\end{proof}

\theoremdet*
\begin{proof}
By unitary invariance of the Frobenius norm, we can drop the leading orthogonal factors: $\widetilde{A} = \Sigma U^\top = U^\top A$, $\widetilde{Y} = \widetilde{A} \Omega$, $\widetilde{G} = U^\top G = \Sigma^{1/2} V$.
\begin{align*}
    \|(1-\proj_{B})&G\|_F^2 = \|(1-\proj_{Y})G\|_F^2\\
    &=\|U^\top(1-\proj_{Y})U \widetilde{G}\|_F^2
    =\|(1-\proj_{\widetilde{Y}})\widetilde{G}\|_F^2.
\end{align*}
Now, note that $\gamma$ being greater than the number of nonzero singular values of $G$ is equivalent to $\Sigma_1$ having some of its diagonal entries be zero. In that case, $\Sigma_2=0$ by the ordering of singular values by magnitude. Then, 
\[\range(\widetilde{G}) = \range(\widetilde{A}) = \range 
\begin{bmatrix}
 \Sigma_1 U_1^\top \\ 0
\end{bmatrix}. \]
$U_1^\top$ is full rank as $U$ is orthogonal, and we assumed $\Omega_1$ has full rank, so $\Omega_1 = U_1^\top \Omega$ implies $U_1^\top, \Omega_1$ have the same range. So:
\begin{equation*}
\range 
\begin{bmatrix}
 \Sigma_1 U_1^\top \\ 0
\end{bmatrix}  = \range \begin{bmatrix}
 \Sigma_1 \Omega_1 \\ 0
\end{bmatrix} = \range(\widetilde{Y}),
\end{equation*}
in which case $\|(1-\proj_{\widetilde{Y}})\widetilde{G}\|_F^2=0$. Else, we can now assume $\Sigma_1$ has strictly non-zero entries. We construct $Z$ by multiplying $\widetilde{Y}$ by a factor to turn the top part into the identity:
\begin{align*}
    \widetilde{Y} &= \widetilde{A}\Omega = \Sigma U^\top \Omega
    = \begin{bmatrix}
        \Sigma_1 U_1^\top \Omega\\
        \Sigma_2 U_2^\top \Omega\\
    \end{bmatrix}
    = \begin{bmatrix}
        \Sigma_1 \Omega_1\\
        \Sigma_2 \Omega_2
    \end{bmatrix}
    ,\\
    Z &= \widetilde{Y} \Omega_1^\dagger \Sigma_1^{-1} 
    = \begin{bmatrix}
        I\\ F 
    \end{bmatrix},
\end{align*}
where $F= \Sigma_2 \Omega_2 \Omega_1^{\dagger} \Sigma_1^{-1}$. By construction, $\range(Z) \subseteq \range(\widetilde{Y})$, so $Z^\perp$ can be split into $Y^\perp$ and an orthogonal subspace: $Z^\perp = Y^\perp \oplus (Y \cap Z^{\perp})$. Therefore, the projection of any vector onto $Z^\perp$ is at least as large as the projection onto $Y^\perp$. Apply this result to every column of $\widetilde{G}$ to get
\begin{equation*}
\|(1-\proj_{\widetilde{Y}})\widetilde{G}\|_F^2 \le \|(1-\proj_{{Z}})\widetilde{G}\|_F^2.
\end{equation*}

We can expand using the fact that the Frobenius norm is invariant to unitary factors:
\begin{align*}
\|(1-\proj_{{Z}})\widetilde{G}\|_F^2
&=\|(1-\proj_{{Z}})\Sigma_G\|_F^2\\
&=\|(1-\proj_{{Z}})\Sigma^{1/2}\|_F^2.
\end{align*}
Lastly, we write the square Frobenius norm as a trace, and use the fact that $1-\proj_{{Z}}$ is a projection matrix:
\begin{align*}
\|(1-\proj_{{Z}})&\Sigma^{1/2}\|_F^2\\
&=\Tr\left([(1-\proj_{{Z}})\Sigma^{1/2}]^\top(1-\proj_{{Z}})\Sigma^{1/2}\right)\\
&=\Tr(\Sigma^{1/2}(1-\proj_{{Z}})\Sigma^{1/2}).
\end{align*}

Now, we note that the expansion of $1-\proj_{{Z}}$ can be written:
\begin{align*}
1-\proj_{{Z}} &=I - Z(Z^\top Z)^{-1} Z^\top\\
 &= \begin{bmatrix}
I - (I + F^\top F)^{-1} & -(I + F^\top F)^{-1}F^\top\\
-F(I + F^\top F)^{-1} & I - F(I + F^\top F)^{-1}F^\top 
\end{bmatrix}.
\end{align*}

To examine the trace of this matrix conjugated by $\Sigma^{1/2}$, we will use the positive semi-definite (PSD) relation $\preccurlyeq$, where $X\preccurlyeq Y$ if and only if $Y-X$ is positive semi-definite.
We abbreviate the upper right block as $M=-(I + F^\top F)^{-1}F^\top$, and therefore the bottom left block becomes $M^\top =-F(I + F^\top F)^{-1}$.

By Proposition 8.2 of \citet{halko}, $I - (I + F^\top F)^{-1} \preccurlyeq F^\top F$. Next, as $F^\top F$ is PSD and symmetric, $(I + F^\top F)^{-1}$ is also PSD and symmetric. So, by the conjugation rule, $F(I + F^\top F)^{-1}F^\top$ is PSD. Therefore $I - F(I + F^\top F)^{-1}F^\top \preccurlyeq I$. So,
\[1-\proj_{{Z}} \preccurlyeq \begin{bmatrix}
 F^\top F & M\\
M^\top & I 
\end{bmatrix}.\]
Conjugating both sides by $\Sigma^{1/2}$ gives
\[\Sigma^{1/2}(1-\proj_{{Z}})\Sigma^{1/2} \preccurlyeq \begin{bmatrix}
 \Sigma^{1/2}_1 F^\top F \Sigma^{1/2}_1 & \Sigma^{1/2}_1 M \Sigma^{1/2}_2\\
\Sigma^{1/2}_2 M^\top \Sigma^{1/2}_1 & \Sigma^{1/2}_2\Sigma^{1/2}_2 
\end{bmatrix}.\]

Now, note that the trace of a PSD matrix is non-negative. So, given $X \preccurlyeq Y$, $\Tr(X)\le \Tr(Y)$. Applying this to the PSD relation above gives:
\begin{align*}
\Tr(\Sigma^{1/2}(1-&\proj_{{Z}})\Sigma^{1/2})\\
&\le \Tr(\Sigma^{1/2}_1 F^\top F \Sigma^{1/2}_1) + \Tr(\Sigma^{1/2}_2\Sigma^{1/2}_2)\\
&=\|F\Sigma^{1/2}_1\|^2_F + \|\Sigma^{1/2}_2\|^2_F.
\end{align*}
Expanding the definition of $F$ achieves the desired result:
\begin{align*}
\mathcal{E}_G(B)
&\le\|F\Sigma^{1/2}_1\|^2_F + \|\Sigma^{1/2}_2\|^2_F\\
&=\|\Sigma_2 \Omega_2 \Omega_1^{\dagger} \Sigma_1^{-1}\Sigma^{1/2}_1\|^2_F + \|\Sigma^{1/2}_2\|^2_F \\
&=\|\Sigma_2 \Omega_2 \Omega_1^{\dagger} \Sigma_1^{-1/2}\|^2_F + \|\Sigma^{1/2}_2\|^2_F.
\end{align*}
\end{proof}

\theoremexp*
\begin{proof}
By Theorem \ref{thm:deterministic-bound} we have:
\begin{equation*}
    \mathcal{E}_G(B) \le \|\Sigma_2 \Omega_2 \Omega_1^\dagger \Sigma_1^{-1/2}\|^2 + \|\Sigma_2^{1/2}\|^2,
\end{equation*}
where $\Omega_2 := U_2^\top \Omega$ for $U_2$ as defined in section $4.2.1$. Taking expectations on both sides,
\[\mathbb{E}[\mathcal{E}_G(B)] \le \mathbb{E}\left[\|\Sigma_2 \Omega_2 \Omega_1^\dagger \Sigma_1^{-1/2}\|^2 + \|\Sigma_2^{1/2}\|^2\right].\]
The expectation can be split over $\Omega_2$ and $\Omega_1$. Now, we use Proposition 10.1 of \citet{halko}, which states that $\mathbb{E}_G\|M_1M_2M_3\|^2 = \|M_1\|^2 \|M_3\|^2$ if $M_2$ is standard normally distributed:
\begin{align*}
\mathbb{E}\left[\|\Sigma_2 \Omega_2 \Omega_1^\dagger \Sigma_1^{-1/2}\|^2 + \|\Sigma_2^{1/2}\|^2\right]\\
&\hspace{-2cm} =\mathbb{E}_{\Omega_1}\left[\|\Sigma_2\|^2 \|\Omega_1^\dagger \Sigma_1^{-1/2}\|^2 + \|\Sigma_2^{1/2}\|^2\right] \\
&\hspace{-2cm} = \Tr(\Sigma_2^2) \mathbb{E}_{\Omega_1} \left[\|\Omega_1^\dagger \Sigma_1^{-1/2}\|^2\right] + \Tr(\Sigma_2).
\end{align*}
Now, we calculate the term $\mathbb{E}_{\Omega_1} \|\Omega_1^\dagger \Sigma_1^{-1/2}\|^2$ separately:
\begin{align*}
\mathbb{E}\|\Omega_1^\dagger \Sigma_1^{-1/2}\|^2 &= \mathbb{E}\left[\Tr ((\Omega_1^\dagger \Sigma_1^{-1/2})^\top \Omega_1^\dagger \Sigma_1^{-1/2})\right] \\
&=\mathbb{E}\left[ \Tr(\Sigma_1^{-1/2} (\Omega_1^\dagger)^\top \Omega_1^\dagger \Sigma_1^{-1/2})\right] \\
&=\mathbb{E}\left[\Tr(\Sigma_1^{-1/2} (\Omega_1 \Omega_1^\top)^{-1} \Sigma_1^{-1/2})\right]\\
&=\Tr(\Sigma_1^{-1/2}\mathbb{E} \left[( \Omega_1 \Omega_1^\top )^{-1}\right]\Sigma_1^{-1/2}).
\end{align*}

Now, note that $W = \Omega_1 \Omega_1^\top$ has a Wishart distribution, namely $W_k(\gamma, I)$, so, by a well-known result of multivariate statistics \citep[pg. 97]{muirhead}, $\mathbb{E}[ W^{-1}] = \frac{1}{k-\gamma-1} I$. Then:
\begin{align*}
\mathbb{E}_{\Omega_1} \|\Omega_1^\dagger \Sigma_1^{-1/2}\|^2 
&=\frac{1}{k-\gamma-1} \Tr(\Sigma_1^{-1/2} I \Sigma_1^{-1/2})\\
&=\frac{1}{k-\gamma-1} \Tr(\Sigma_1^{-1}).
\end{align*}

This result allows for the expansion of the $\mathbb{E}_{\Omega_1} \|\Omega_1^\dagger \Sigma_1^{-1/2}\|^2$ term:
\begin{align*}
    \Tr (\Sigma_2^2) \mathbb{E}_{\Omega_1} \|\Omega_1^\dagger &\Sigma_1^{-1/2}\|^2 + \Tr(\Sigma_2)\\
    &= \Tr (\Sigma_2^2) \Tr(\Sigma_1^{-1})\cdot (\tfrac{\gamma}{k-\gamma-1}) + \Tr(\Sigma_2).
\end{align*}

Finally, to finish we note that the choice of $\gamma$ was arbitrary except for the condition $\gamma \le k-2$, so taking the minimum over all such $\gamma$ creates a tighter upper bound:
\begin{equation*}
    \mathbb{E}[\mathcal{E}_G(B)] \le \min_{\gamma\le k-2}\tfrac{\gamma}{k-\gamma-1}\Tr(\Sigma_2^{2}) \Tr(\Sigma_1^{-1}) + \Tr(\Sigma_2).
\end{equation*}
\end{proof}

\theoremthree*
\begin{proof}
For matrices $Q,X$ described in Sketching Method \ref{alg:method3}, $\range(B_\mathrm{sym}) = \range([Q,X^\top]) = \range(B) + \range(X^\top)$. So, the range of $B$ is just a subset of the range of the symmetric approximation $B_\mathrm{sym}$, which implies:
\[|(1-\proj_{B_\mathrm{sym}})u|_2 \le |(1-\proj_{B})u|_2.\]
By squaring and summing over the columns of $G$, the result expands to:
\begin{align*}
\|(1-\proj_{B_\mathrm{sym}})G\|_F^2 &\le \|(1-\proj_{B})G\|_F^2,
\end{align*}
which implies $\mathcal{E}_G(B_\mathrm{sym}) \le \mathcal{E}_G(B)$. This completes the proof.
\end{proof}

\rev{
\theoremconv*
\begin{proof}
We first characterize $w_\tau^*$ as the solution of an optimization and then prove that the parameter $w_\tau^{(j)}$ actually converges to it.
At each task $\tau$, SketchOGD tries to minimize the loss function~\eqref{eq:loss} while updating the parameter orthogonally to the columns of $B_\tau$. Thus, its goal can be presented with the following optimization problem:
\begin{equation}\label{eq:opt_sketchogd}
\begin{aligned}
    w_{\tau}^* = 
    \argmin_{w\in\mathbb{R}^p} & \frac{1}{n_\tau} \|f(\mathtt{X}_\tau,w)-\mathtt{Y}_\tau\|_2^2 + \lambda \|w-w_{\tau-1}^*\|_2^2\\
    \text{ s.t. } & B_\tau^\top (w-w_{\tau-1}^*)=0
\end{aligned},
\end{equation}
while the error term in the objective can be reformulated as
\begin{align*}
    f(\mathtt{X}_\tau,w)\!-\!\mathtt{Y}_\tau &= f(\mathtt{X}_\tau,w) \!-\! f(\mathtt{X}_\tau,w_{\tau\!-\!1}^*)\!-\! (\mathtt{Y}_\tau \!-\! f(\mathtt{X}_\tau,w_{\tau\!-\!1}^*))\\
    &= G_\tau^\top(w-w_{\tau-1}^*)-\tilde{\mathtt{Y}}_\tau.
\end{align*}
For $b_\tau:=\rank(B_\tau)$, consider $\bar{B}_\tau\in\mathbb{R}^{p\times(p-b_\tau)}$ as an orthogonal complement of $B_\tau$, which satisfies $T_\tau=I_p-B_\tau B_\tau^\top=\bar{B}_\tau\bar{B}_\tau^\top$. Then, the constraint in~\eqref{eq:opt_sketchogd} has equivalence:
\begin{equation*}
    B_\tau^\top (w-w_{\tau-1}^*)=0 \iff \exists\tilde{w}\in\mathbb{R}^{p-b_\tau} \text{ s.t. } w-w_{\tau-1}^* = \bar{B}_\tau \tilde{w}.
\end{equation*}
Then, \eqref{eq:opt_sketchogd} is equivalent to solving the following unconstrained optimization with $w_{\tau}^*=\bar{B}_\tau\tilde{w}_\tau^* + w_{\tau-1}^*$:
\begin{equation}\label{eq:opt_sketchogd_unconst}
    \tilde{w}_\tau^*=
    \argmin_{\tilde{w}\in\mathbb{R}^{p-b_\tau}} \frac{1}{n_\tau} \|G_\tau^\top \bar{B}_\tau\tilde{w} -\tilde{\mathtt{Y}}_\tau\|_2^2 + \lambda \|\bar{B}_\tau\tilde{w}\|_2^2.
\end{equation}
Since $\bar{B}_\tau^\top\bar{B}_\tau=I_{p-b_\tau}$, \eqref{eq:opt_sketchogd_unconst} is equivalent to a ridge regression problem:
\begin{equation}\label{eq:opt_tilde}
    \tilde{w}_\tau^*=
    \argmin_{\tilde{w}\in\mathbb{R}^{p-b_\tau}} \frac{1}{n_\tau} \|G_\tau^\top \bar{B}_\tau\tilde{w} -\tilde{\mathtt{Y}}_\tau\|_2^2 + \lambda \|\tilde{w}\|_2^2,
\end{equation}
which has a unique solution:
\begin{align*}
    \tilde{w}_\tau^* &=
    (\bar{B}_\tau^\top G_\tau G_\tau^\top \bar{B}_\tau+ n_\tau \lambda I_{p-b_\tau})^{-1} \bar{B}_\tau^\top G_\tau \tilde{\mathtt{Y}}_\tau\\
    &= \bar{B}_\tau^\top G_\tau (G_\tau^\top \bar{B}_\tau \bar{B}_\tau^\top G_\tau + n_\tau \lambda I_{p-b_\tau})^{-1} \tilde{\mathtt{Y}}_\tau.
\end{align*}

Now, we show $w_\tau^{(j)}$ converges to $w_\tau^*=\bar{B}_\tau\tilde{w}_\tau^* + w_{\tau-1}^*$ if the step-size $\eta_\tau$ is small enough. First,
\begin{equation}\label{eq:solution_gap}
\begin{aligned}
    w_\tau^{(j+1)} \!-\! w_\tau^* &= w_\tau^{(j)} - \eta_\tau T_\tau \nabla \mathcal{L}_\tau(w_\tau^{(j)}) - w_\tau^*\\
    &= \big(I - \eta_\tau T_\tau(\frac{2}{n_\tau}G_\tau G_\tau^\top + 2\lambda I_p)\big)(w_\tau^{(j)} \!-\! w_{\tau-1}^*)\\
    & \hspace{2cm} + \frac{2}{n_\tau}\eta_\tau T_\tau G_\tau \tilde{\mathtt{Y}}_\tau + w_{\tau-1}^* - w_\tau^*\\
    &= \big(I - \eta_\tau T_\tau(\frac{2}{n_\tau}G_\tau G_\tau^\top + 2\lambda I_p)\big) (w_\tau^{(j)} - w_{\tau}^*),
\end{aligned}    
\end{equation}
where the last equality is from the optimality of $\tilde{w}_\tau^*$ in~\eqref{eq:opt_tilde} such that it makes the gradient of the objective zero:
\begin{equation*}
    0 = (\frac{2}{n_\tau} \bar{B}_\tau^\top G_\tau G_\tau^\top \bar{B}_\tau + 2 \lambda I_{p-b_\tau})\tilde{w}_\tau^* - \frac{2}{n_\tau} \bar{B}_\tau^\top G_\tau \tilde{\mathtt{Y}}_\tau,
\end{equation*}
which implies
\begin{align*}
    \frac{2}{n_\tau}\eta_\tau T_\tau G_\tau \tilde{\mathtt{Y}}_\tau
    &= \eta_\tau \bar{B}_\tau (\frac{2}{n_\tau} \bar{B}_\tau^\top G_\tau G_\tau^\top \bar{B}_\tau + 2 \lambda I_{p-b_\tau})\tilde{w}_\tau^*\\
    &= \eta_\tau \bar{B}_\tau \bar{B}_\tau^\top (\frac{2}{n_\tau} G_\tau G_\tau^\top + 2 \lambda I_{p})\bar{B}_\tau \tilde{w}_\tau^*\\
    &= \eta_\tau T_\tau (\frac{2}{n_\tau} G_\tau G_\tau^\top + 2 \lambda I_{p}) (w_{\tau}^*-w_{\tau-1}^*).
\end{align*}

Then, the recursive expression in~\eqref{eq:solution_gap} converges if and only if the spectral radius of $I - \eta_\tau T_\tau(\frac{2}{n_\tau}G_\tau G_\tau^\top + 2\lambda I_p)$ is less than 1. The condition is equivalent to
\begin{equation}
    0<\eta_\tau<\dfrac{1}{\lambda_{max}\big(T_\tau(\frac{1}{n_\tau}G_\tau G_\tau^\top+\lambda I_p)\big)}.
\end{equation}

\end{proof}
}

\begin{figure}[t]
\centering
  \includegraphics[width=\columnwidth]{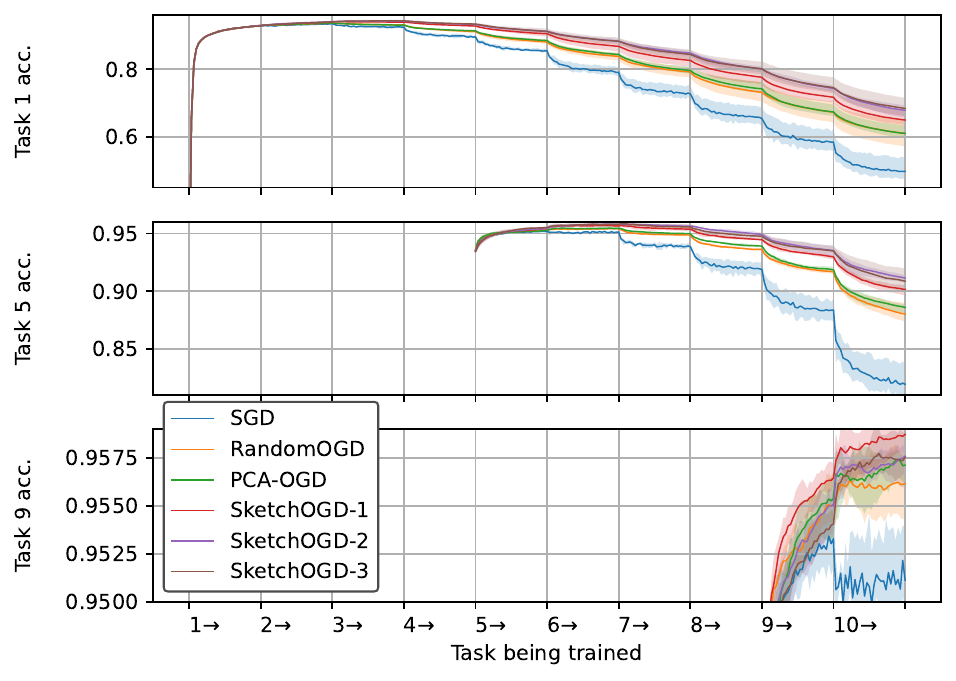}
  \caption{Accuracy of various continual learning methods on tasks 1, 5, and 9 of the Rotated MNIST benchmark. Each plot shows the initial learning of a task and the subsequent slow forgetting as the model trains on new tasks. As shown, SketchOGD-2,3 perform the best in this scenario.}
  \label{fig:rotated_graph_large}
\end{figure}

\begin{figure}[t]
  \centering
  \includegraphics[width=\columnwidth]{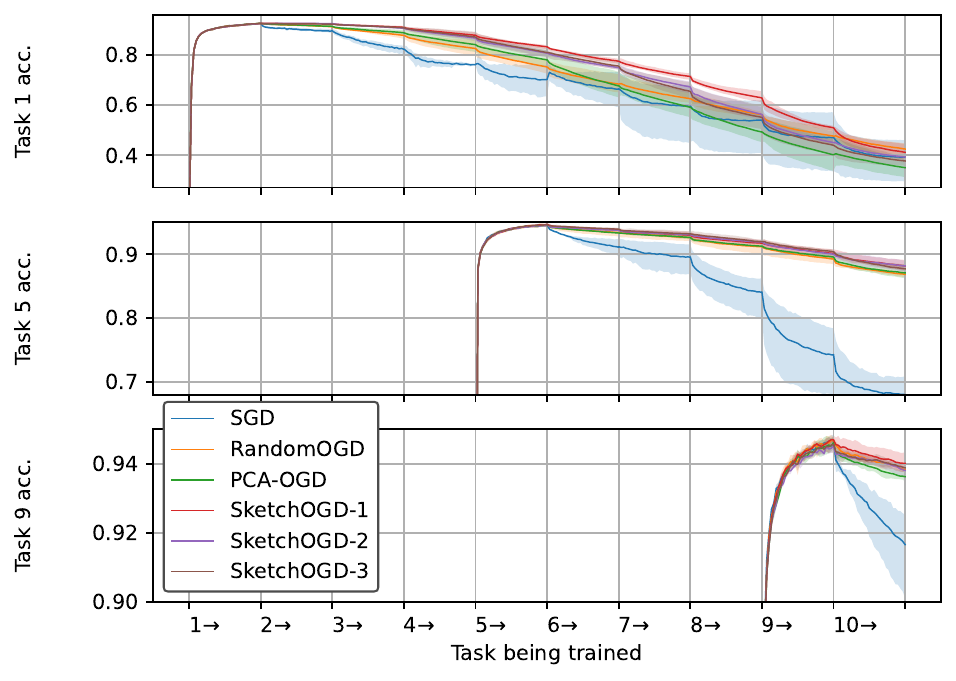}
  \caption{Accuracy of various continual learning methods on tasks 1, 5, and 9 of the Permuted MNIST benchmark. Each plot shows the initial learning of a task and the subsequent slow forgetting as the model trains on new tasks. As shown, SketchOGD-1 performs the best in this scenario. Interestingly, SGD experiences relatively less catastrophic forgetting on the first task compared to later tasks.}
  \label{fig:permuted_graph}
\end{figure}

\begin{figure}[t]
\centering
  \includegraphics[width=\columnwidth]{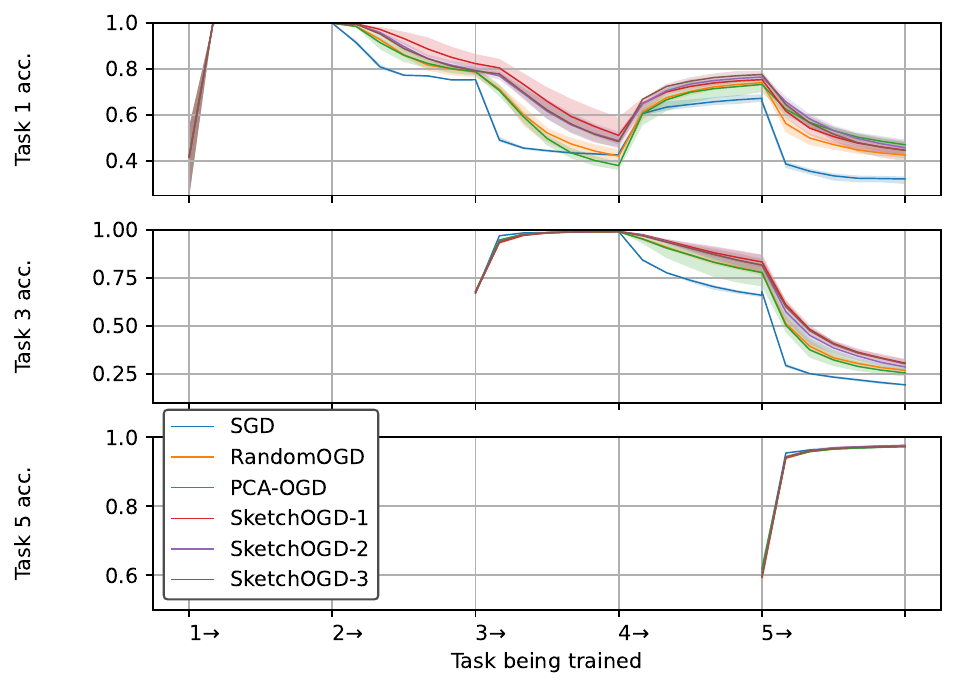}
  \caption{Accuracy of various continual learning methods on tasks 1, 3, and 5 of the Split MNIST benchmark. Each plot shows the initial learning of a task and the subsequent slow forgetting as the model trains on new tasks. The OGD variants show similar performance in this scenario while they outperform SGD.}
  \label{fig:split_graph}
\end{figure}

\begin{figure}[t]
\centering
  \includegraphics[width=\columnwidth]{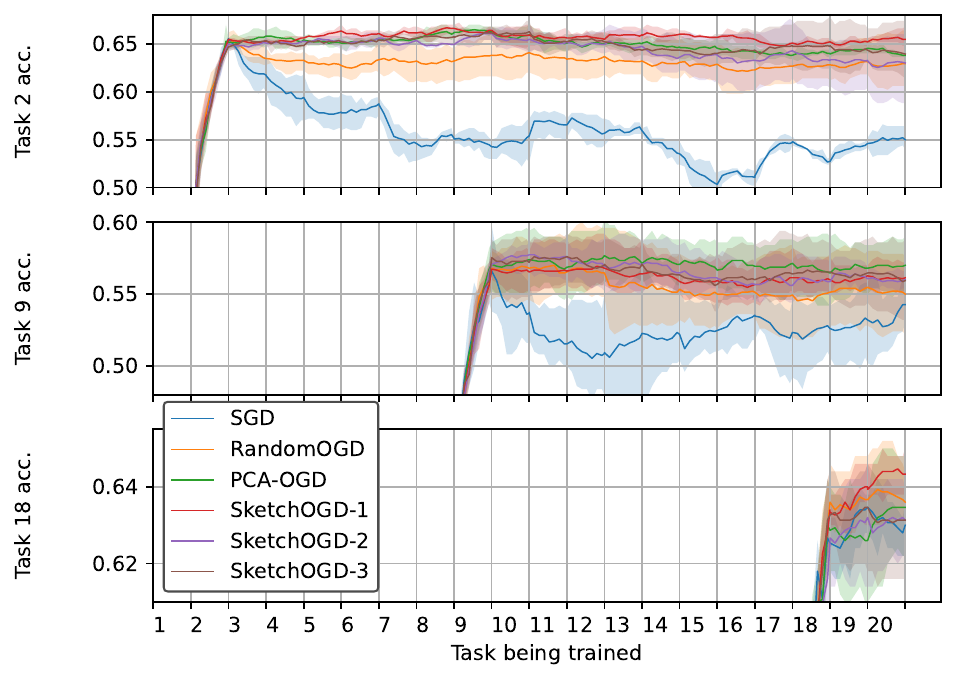}
  \caption{Accuracy of various continual learning methods on tasks 2, 9, and 18 of the Split CIFAR benchmark. Each plot shows the initial learning of a task and the subsequent slow forgetting as the model trains on new tasks. The OGD variants show similar performance in this scenario while they outperform SGD.}
  \label{fig:split_graph_cifar}
\end{figure}

\begin{figure}[t]
\centering
  \includegraphics[width=\columnwidth]{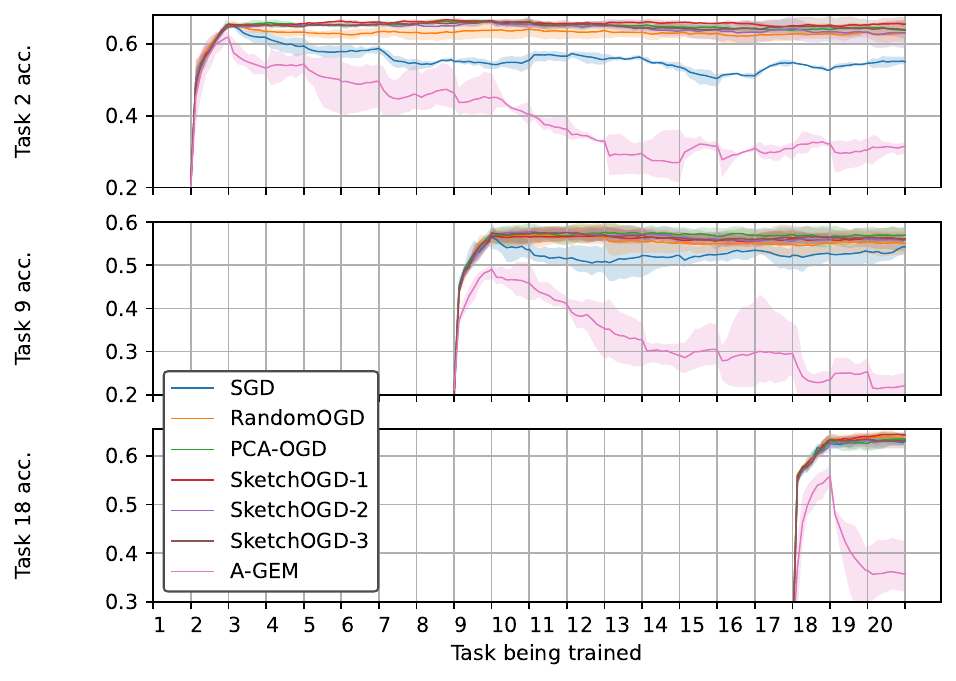}
  \caption{Accuracy of various continual learning methods including A-GEM on tasks 2, 9, and 18 of the Split CIFAR benchmark. Each plot shows the initial learning of a task and the subsequent slow forgetting as the model trains on new tasks. As shown, A-GEM experiences significantly more catastrophic forgetting than other methods including SGD.}
  \label{fig:split_graph_cifar_agem}
\end{figure}

\subsection{Additional Experimental Results}

In this section, we provide an extended list of plots for visualizing the experimental results. First, we show the performance of SketchOGD compared to other OGD variants in Rotated MNIST, Permuted MNIST, Split MNIST, and Split CIFAR in Figures \ref{fig:rotated_graph_large}-\ref{fig:split_graph_cifar}. Then, we also plot A-GEM's performance on Split CIFAR in Figure \ref{fig:split_graph_cifar_agem}. A-GEM performs significantly worse in Split CIFAR than any OGD variant, suggesting that the single-gradient approach of A-GEM is ineffective for larger models.

\end{document}